\documentclass{article}

\usepackage{microtype}
\usepackage{graphicx}
\usepackage{subfigure}
\usepackage{booktabs} % for professional tables

\usepackage[bibliography=common]{apxproof}
\renewcommand{\appendixprelim}{}
\newtheoremrep{lemma}{Lemma}
\newtheorem{definition}{Definition}
\newtheoremrep{proposition}{Proposition}
\newtheoremrep{theorem}{Theorem}

\usepackage{amsmath}
\usepackage{amssymb}
\newcommand{\E}{\mathbb{E}}
\newcommand{\R}{\mathbb{R}}
\DeclareMathOperator*{\argmax}{arg\,max}

\usepackage{hyperref}

\usepackage[accepted]{icml2019}

\icmltitlerunning{Probability Functional Descent}

\begin{document}

\twocolumn[
\icmltitle{Probability Functional Descent: A Unifying Perspective\\ on GANs, Variational Inference, and Reinforcement Learning}




\begin{icmlauthorlist}
\icmlauthor{Casey Chu}{stanfordicme}
\icmlauthor{Jose Blanchet}{stanfordmse} 
\icmlauthor{Peter Glynn}{stanfordmse}
\end{icmlauthorlist}

\icmlaffiliation{stanfordicme}{Institute for Computational \& Mathematical Engineering, Stanford University, Stanford, California, USA}
\icmlaffiliation{stanfordmse}{Management Science \& Engineering, Stanford University, Stanford, California, USA}

\icmlcorrespondingauthor{Casey Chu}{caseychu@stanford.edu}

\icmlkeywords{Machine Learning, ICML}

\vskip 0.3in
]



\printAffiliationsAndNotice{}  % leave blank if no need to mention equal contribution

\begin{abstract}
The goal of this paper is to provide a unifying view of a wide range of problems of interest in machine learning by framing them as the minimization of functionals defined on the space of probability measures. In particular, we show that generative adversarial networks, variational inference, and actor-critic methods in reinforcement learning can all be seen through the lens of our framework. We then discuss a generic optimization algorithm for our formulation, called \emph{probability functional descent} (PFD), and show how this algorithm recovers existing methods developed independently in the settings mentioned earlier.
\end{abstract}

\section{Introduction}

\begin{table*}
    \centering
    \begin{tabular}{lccc} \toprule
         Domain & Distribution of interest & Functional & Functional derivative \\\midrule
        Generative adversarial networks & Generator $\mu$ & $D(\mu || \nu)$ & Discriminator $D^*(x)$ \\
        Variational inference & Approximate posterior $q(z)$ & $D_{\text{KL}}(q(z) || p(z|x))$ & Negative ELBO $\log \frac{q(z)}{p(x,z)}$ \\
         Reinforcement learning & Policy $\pi(a|s)$ & Expected reward & Advantage $Q^\pi(s,a) - V^\pi(s)$ \\\bottomrule
    \end{tabular}
    \caption{Framing a problem as the optimization of a probability functional unifies several domains.}
    \label{fig:domains}
\end{table*}

\begin{table*}
    \centering
    \begin{tabular}{lcccc} \toprule
         Algorithm & Type of derivative estimator  \\\midrule
        \textbf{Generative adversarial networks} \\
         Minimax GAN \cite{goodfellow2014generative} & Convex duality \\
         Non-saturating GAN \cite{goodfellow2014generative} & Binary classification \\
         Wasserstein GAN \cite{arjovsky2017wasserstein} & Convex duality \\\midrule
        \textbf{Variational inference} \\
         Black-box variational inference \cite{ranganath2014black} & Exact \\
         Adversarial variational Bayes \cite{mescheder2017adversarial} & Binary classification \\
         Adversarial posterior distillation \cite{wang2018adversarial} & Convex duality \\\midrule
        \textbf{Reinforcement learning} \\
        Policy iteration \cite{howard1960dynamic} & Exact \\
        Policy gradient \cite{williams1992simple} & Monte Carlo \\
        Actor-critic \cite{konda2000actor,sutton2000policy} & Least squares \\
        Dual actor-critic \cite{chen2016stochastic,dai2017boosting} & Convex duality \\ \bottomrule
    \end{tabular}
    \caption{Different existing algorithms correspond to different ways of estimating the functional derivative.}
    \label{fig:approximations}
\end{table*}

Deep learning now plays an important role in many domains, for example, in generative modeling, deep reinforcement learning, and variational inference. In the process, dozens of new algorithms have been proposed for solving these problems with deep neural networks, specific of course to domain at hand. 

In this paper, we introduce a conceptual framework which can be used to understand in a unified way a broad class of machine learning problems. Central to this framework is an abstract optimization problem in the space of probability measures, a formulation that stems from the observation that in many fields, the object of interest is a probability distribution; moreover, the learning process is guided by a \emph{probability functional} to be minimized, a loss function that conceptually maps a probability distribution to a real number. \autoref{fig:domains} lists these correspondences in the case of generative adversarial networks, variational inference, and reinforcement learning.

Because the optimization now takes place in the infinite-dimensional space of probability measures, standard finite-dimensional algorithms like gradient descent are initially unavailable; even the proper notion for the derivative of these functionals is unclear. We call upon on a body of literature known as von Mises calculus \cite{mises1947asymptotic,fernholz2012mises}, originally developed in the field of asymptotic statistics, to make these functional derivatives precise. Remarkably, we find that once the connection is made, the resulting generalized descent algorithm, which we call \emph{probability functional descent}, is intimately compatible with standard deep learning techniques such as stochastic gradient descent \cite{bottou2010large}, the reparameterization trick \cite{kingma2013auto}, and adversarial training \cite{goodfellow2014generative}. 

When we apply probability functional descent to the aforementioned domains, we find that we recover a wide range of existing algorithms, and the essential distinction between them is simply the way that the functional derivative, the \emph{von Mises influence function} in this context, is approximated. \autoref{fig:approximations} lists these algorithms and their corresponding approximation methods. Probability functional descent therefore acts as a unifying framework for the analysis of existing algorithms as well as the systematic development of new ones.

\subsection{Related work}
The problem of optimizing functionals of probability measures is not new. For example, \citet{gaivoronski1986linearization} and \citet{molchanov2001variational} study these types of problems and even propose Frank-Wolfe and steepest descent algorithms to solve these problems. However, their algorithms are not immediately practical for the high-dimensional machine learning settings described here, and it is not clear how to integrate their methods with modern deep learning techniques.

Several others in the machine learning community also adopt the perspective of descent in the space of probability distributions. In order to introduce functional gradients, these approaches endow the space of probability distributions with either Hilbert structure \cite{dai2014scalable,dai2016provable,liu2016stein,dai2018learning} or Wasserstein structure \cite{richemond2017wasserstein,frogner2018approximate,zhang2018policy,lin2018wasserstein} and rely on gradient descent or Wasserstein gradient flow respectively to decrease the objective value. Such approaches typically require kernel-based or particle-based methods to implement in practice. By contrast, our approach foregoes gradients and instead directly considers descent on linear approximations by leveraging the G\^ateaux derivative. As we shall illustrate, this approach is more compatible with standard deep learning techniques and indeed leads exactly to many existing deep learning-based algorithms. \citet{carmona2018probabilistic} provide a technical comparison between these differing approaches for defining derivatives in chapter 5.

%Finally, we mention the work of \citet{dai2018learning}, who also casts many machine learning problems, including Bayesian inference and policy evaluation, as infinite-dimensional optimization problems. These optimization problems are over functions in Hilbert spaces with an objective in the form of an integral operator. Because probability distributions may often be represented as functions in an appropriate Hilbert space (via their density, for example, if they exist), this approach bears similarity to ours. Their use of a Hilbert space as the optimization domain represents a tradeoff that allows access to powerful Hilbert space structure but also requires the use of specialized kernel-based \cite{dai2014scalable} or particle-based methods \cite{dai2016provable,liu2016stein} to practically apply functional gradients and obtain samples. In our work, we work directly in the space of probability measures and associate a dual space without assuming any Hilbert (or even Banach) structure. This approach, as we shall illustrate, is compatible with standard deep learning techniques and indeed corresponds exactly to many existing deep learning-based algorithms.

Finally, one part of our work recasts convex optimization problems as saddle-point problems by means of convex duality as a technique for estimating functional derivatives. This correspondence between convex optimization problems and saddle point problems is an old and general concept \cite{rockafellar1968general}, and it underlies classical dual optimization techniques \cite{lucchetti2006convexity,luenberger2008linear}. Nevertheless, the use of these min-max representations remains an active topic of research in machine learning. Most notably, the literature concerning generative adversarial networks has recognized that certain min-max problems are equivalent to certain convex problems \cite{goodfellow2014generative,nowozin2016f,farnia2018convex}. Outside of GANs, \citet{dai2017learning,dai2018sbeed} have begun using these min-max representations to inspire learning algorithms. These min-max representations are an important tool for us that allows for practical implementation of our theory.

\section{Descent on a Probability Functional}
We let $\mathcal{P}(X) $ be the space of Borel probability measures on a topological space $X$. Our abstract formulation takes the form of a minimization problem over probability distributions:
\[
    \min_{\mu \in \mathcal{P}(X)} J(\mu),
\]
where $J:\mathcal{P}(X) \to \R$ is called a probability functional. In order to avoid technical digressions, we assume that $X$ is a metric space that is compact, complete, and separable (i.e.~a compact Polish space). We endow $\mathcal{P}(X)$ with the topology of weak convergence, also known as the weak* topology.

We now draw upon elements of von Mises calculus \cite{mises1947asymptotic} to make precise the notion of derivatives of functionals such as $J$. See \citet{fernholz2012mises} for an in-depth discussion, or \citet{santambrogio2015functionals} for another perspective.

\begin{definition}[(G\^ateaux differential)]
Let $J:\mathcal{P}(X) \to \R$ be a function. The G\^ateaux differential $dJ_\mu$ at $\mu \in \mathcal{P}(X)$ in the direction $\chi$ is defined by \begin{equation}
    dJ_\mu(\chi) = \lim_{\epsilon \to 0^+} \frac{J(\mu + \epsilon \chi) - J(\mu)}{\epsilon}, \label{def:gateaux}
\end{equation} where $\chi = \nu - \mu$ for some $\nu \in \mathcal{P}(X)$.
\end{definition}

Intuitively, the G\^ateaux differential is a generalization of the directional derivative, so that $dJ_\mu(\chi)$ describes the change in the value of $J(\mu)$ when the probability measure $\mu$ is infinitesimally perturbed in the direction of $\chi$, towards another measure $\nu$. Though powerful, the G\^ateaux differential is a function of differences of probability measures, which can make it unwieldy to work with. In many cases, however, the G\^ateaux differential $dJ_\mu(\chi)$ can be concisely represented as an integral of an \emph{influence function} $\Psi_\mu : X \to \R$, where the integral is taken with respect to the measure $\chi$.
\begin{definition}[(Influence function)]
We say that $\Psi_\mu : X \to \R$ is an influence function for $J$ at $\mu \in \mathcal{P}(X)$ if the G\^ateaux differential $dJ_\mu(\chi)$ has the integral representation \begin{equation}
    dJ_\mu(\chi) = \int_X \Psi_{\mu}(x) \,\chi(dx) \label{def:influence}
\end{equation} for all $\chi = \nu - \mu$, where $\nu \in \mathcal{P}(X)$.
\end{definition}

\begin{toappendix}
\begin{lemma} \label{thm:influence} Let $J : \mathcal{P}(X) \to \R$. Then $\Psi : X \to \R$ is an influence function of $J$ at $\mu$ if and only if \[
    \frac{d}{d\epsilon}J(\mu +\epsilon \chi ) \Big|_{\epsilon=0^+} = \int_X \Psi(x)\,\chi(dx).
\]
\end{lemma}
\begin{proof}
    The left-hand side equals \eqref{def:gateaux}, which equals \eqref{def:influence}.
\end{proof}
\end{toappendix}

The influence function provides a convenient representation for the G\^ateaux differential. Because $\chi = \nu - \mu$ is a difference of probability distributions, we can also write \[
    dJ_\mu(\chi) = \E_{x \sim \nu}[ \Psi_\mu(x) ] - \E_{x \sim \mu} [ \Psi_\mu(x)]
\]
by linearity. We note that if $\Psi_\mu$ is an influence function, then so is $\Psi_\mu + c$ for a constant $c$.
 
The G\^ateaux derivative and the influence function provide the proper notion of a functional derivative, which allows us to generalize first-order descent algorithms to apply to probability functionals such as $J$. In particular, they permit a linear approximation to $J(\mu)$ around $\mu_0$, which we denote $\tilde{J}(\mu)$: \begin{align*}
    \tilde{J}(\mu) 
        &= J(\mu_0) + dJ_{\mu_0}(\mu - \mu_0) \\
        &= J(\mu_0) + \E_{x \sim \mu}[ \Psi_{\mu_0}(x) ] - \E_{x \sim \mu_0} [ \Psi_{\mu_0}(x)] \\
        &= \text{constant} + \E_{x \sim \mu}[ \Psi_{\mu_0}(x) ].
\end{align*}

This expression, also known as a von Mises representation, yields additional intuition about the influence function. Concretely, note that a small pertubation to $\mu$ decreases $J(\mu)$ if it decreases $\E_{x \sim \mu}[ \Psi_{\mu_0}(x)]$. Therefore, $\Psi_{\mu_0}$ acts as a potential function defined on $X$ that dictates where samples $x \sim \mu$ should descend if the goal is to decrease $J(\mu)$. Of course, $\Psi_{\mu_0}$ only carries this interpretation around the current value of $\mu_0$.

Based on this intuition, we now present \textbf{probability functional descent}, a straightforward analogue of finite-dimensional first-order descent algorithms to probability functionals. First, a linear approximation to the functional $J$ is computed at $\mu_0$ in the form of the influence function $\Psi_{\mu_0}$, and then a local step is taken from $\mu_0$ so as to decrease the value of the linear approximation. Concretely:

\begin{algorithm}
\caption{Probability functional descent on $J(\mu)$}
\begin{algorithmic}
    \STATE Initialize $\mu$ to a distribution in $\mathcal{P}(X)$
    \WHILE{$\mu$ has not converged}
        \STATE Set $\hat\Psi \approx \Psi_\mu$ (differentiation step)
        \STATE Update $\mu$ to decrease $\E_{x \sim \mu} [\hat\Psi(x)]$ (descent step)
    \ENDWHILE
\end{algorithmic}
\end{algorithm}

We shall see that probability functional descent serves as a blueprint for many existing algorithms: in generative adversarial networks, the differentiation and descent steps correspond to the discriminator and generator updates respectively; in reinforcement learning, they correspond to policy evaluation and policy improvement.

In its abstract form, probability functional descent requires two design choices in order to convert it into a practical algorithm. In \autoref{sec:descent}, we discuss different ways to choose the update in the descent step; \autoref{thm:chain} provides one generic way. In \autoref{sec:derivative}, we discuss different ways to approximate the influence function in the differentiation step; \autoref{thm:fenchel} provides one generic way and an unexpected connection to adversarial training. 

\section{Applying the Descent Step}
\label{sec:descent}
One straightforward way to apply the descent step of PFD is to adopt a parametrization $\theta \mapsto \mu_\theta$ and descend the stochastic gradient of $\theta \mapsto \E_{x \sim \mu_\theta}[ \hat\Psi(x)]$.\footnote{Note that this gradient step is simply one possible choice of update rule for the descent step of PFD; see \autoref{sec:policyiteration} (policy iteration) for an instance of PFD where this gradient-based update rule is not adopted.} This gradient step is justified by the following analogue of the chain rule:
\begin{theoremrep}[(Chain rule)] \label{thm:chain}
Let $J : \mathcal{P}(X) \to \R$ be continuously differentiable, in the sense that the influence function $\Psi_\mu$ exists and $(\mu, \nu) \mapsto \E_{x\sim\nu} [\Psi_\mu(x)]$ is continuous. Let the parameterization $\theta \mapsto \mu_\theta$ be differentiable, in the sense that $\frac{1}{||h||}(\mu_{\theta + h} - \mu_\theta)$ converges to a weak limit as $h \to 0$. Then \[
    \nabla_\theta J(\mu_\theta) = \nabla_\theta \E_{x \sim \mu_\theta} [\hat\Psi(x)],
\] where $\hat\Psi = \Psi_{\mu_\theta}$ is treated as a function $X \to \R$ that is not dependent on $\theta$. 
\end{theoremrep}
\begin{proof}
    Without loss of generality, assume $\theta \in \R$, as the gradient is simply a vector of one-dimensional derivatives. Let $\chi_\epsilon = \frac{1}{\epsilon}(\mu_{\theta +\epsilon} - \mu_\theta)$, and let $\chi = \lim_{\epsilon\to0}\chi_\epsilon$ (weakly). Then \begin{align*}
        \frac{d}{d\theta} J(\mu_\theta)
            &= \frac{d}{d\epsilon} J(\mu_{\theta + \epsilon}) \Big|_{\epsilon=0} \\
            &= \frac{d}{d\epsilon} J(\mu_\theta + \epsilon\chi_\epsilon) \Big|_{\epsilon=0}.
    \end{align*}
    Assuming for now that \[
         \frac{d}{d\epsilon} J(\mu_\theta + \epsilon\chi_\epsilon) \Big|_{\epsilon=0} =  
         \frac{d}{d\epsilon} J(\mu_\theta + \epsilon\chi) \Big|_{\epsilon=0},
    \] we have by \autoref{thm:influence} that \begin{align*}
        \frac{d}{d\theta} J(\mu_\theta) 
            &= \int_X \hat\Psi \,d\chi  \\
            &= \int_X \hat\Psi \,d\Big( \lim_{\epsilon\to 0}  \frac{1}{\epsilon}(\mu_{\theta +\epsilon} - \mu_\theta) \Big) \\
            &= \lim_{\epsilon\to 0} \int_X \hat\Psi \,d\Big(  \frac{1}{\epsilon}(\mu_{\theta +\epsilon} - \mu_\theta) \Big) \\
            &= \frac{d}{d\theta} \int_X \hat\Psi \,d\mu_\theta,
    \end{align*} where the interchange of limits is by the definition of weak convergence (recall we assumed that $X$ is compact, so $\hat\Psi$ is continuous and bounded by virtue of being continuous).
    
    The equality we assumed is the definition of a stronger notion of differentiability called Hadamard differentiability of $J$. Our conditions imply Hadamard differentiability via Proposition 2.33 of \citet{penot2012calculus}, noting that the map $(\mu, \chi) \mapsto \int_X \Psi_\mu\,d\chi$ is continuous by assumption.
\end{proof}

\autoref{thm:chain} converts the computation of $\nabla_{\theta} J(\mu_\theta)$, where $J$ may be a complicated nonlinear functional, into the computation of a gradient of an expectation, which is easily handled using standard methods (see e.g.~\citet{schulman2015gradient}). For example, the reparameterization trick, also known as the pathwise derivative estimator \cite{kingma2013auto,rezende2014stochastic}, uses the identity \[
    \nabla_\theta \E_{x\sim \mu_\theta}[\hat\Psi(x)] = \nabla_\theta \E_{z \sim \mathcal{N}(0,I)}[\hat\Psi(h_\theta(z))],
\] where $\mu_\theta$ samples $x = h_\theta(z)$ using $z \sim \mathcal{N}(0, I)$. Alternatively, the log derivative trick, also known as the score function gradient estimator, likelihood ratio gradient estimator, or REINFORCE \cite{glynn1990likelihood,williams1992simple,kleijnen1996optimization}, uses the identity\[
    \nabla_\theta \E_{x \sim \mu_\theta}[ \hat\Psi(x)] = \E_{x \sim \mu_\theta}[ \hat\Psi(x) \nabla_\theta \log \mu_\theta(x) ],
\] where $\mu_\theta(x)$ is the probability density function of $\mu_\theta$. This gradient-based update rule for the descent step is therefore a natural, practical choice in the context of deep learning.

\section{Approximating the Influence Function}
\label{sec:derivative}

The approximation of the influence function in the differentiation step can in principle be accomplished in many different ways. Indeed, we shall see that the distinguishing factor between many existing algorithms is exactly which influence function estimator used, as shown in \autoref{fig:approximations}. In some cases, it is possible that the influence function can be evaluated exactly, bypassing the need for approximation. Otherwise, the influence function, being a function $X \to \R$, may be modeled as a neural network; the precise way in which this neural network needs to be trained will depend on the exact analytical form of the influence function.

Remarkably, a generic approximation technique is available if the functional $J$ is convex. In this case, the influence function $\Psi_\mu$ possesses a variational characterization in terms of the convex conjugate $J^\star$ of $J$. To apply this formalism, we now view $\mathcal{P}(X)$ as a convex subset of the vector space of finite signed Borel measures $\mathcal{M}(X)$, equipped with the topology of weak convergence. Crucial to the analysis will be its dual space, $\mathcal{C}(X)$, the space of continuous functions $X \to \R$. Finally, $\overline{\R}$ denotes the extended real line $\R \cup \{ -\infty, \infty \}$. The convex conjugate is then defined as follows:
\begin{definition}
Let $J : \mathcal{M}(X) \to \overline{\R}$ be a function. Its convex conjugate is a function $J^\star : \mathcal{C}(X) \to \overline{\R}$ defined by \[
    J^\star(\varphi) = \sup_{\mu \in \mathcal{M}(X)} \Big[\int_X \varphi(x) \,\mu(dx) - J(\mu) \Big].
\]
\end{definition}

Note that $J$ must now be defined on all of $\mathcal{M}(X)$; it is always possible to simply define $J(\mu) = \infty$ if $\mu \not\in \mathcal{P}(X)$, although sometimes a different extension may be more convenient. The convex conjugate forms the core of the following representation for the influence function $\Psi_\mu$:

\begin{theoremrep}[(Fenchel--Moreau representation)] \label{thm:fenchel}
Let $J : \mathcal{M}(X) \to \overline\R$ be proper, convex, and lower semicontinuous. Then the maximizer of $\varphi \mapsto \E_{x \sim \mu} [\varphi(x)] - J^\star(\varphi)$, if it exists, is an influence function for $J$ at $\mu$. With some abuse of notation, we have that \[
    \Psi_\mu = \argmax_{\varphi \in \mathcal{C}(X)} \Big[ \E_{x \sim \mu}[\varphi(x)] - J^\star(\varphi) \Big].
\]
\end{theoremrep}
\begin{proof}
    We will exploit the Fenchel--Moreau theorem, which applies in the setting of locally convex, Hausdorff topological vector spaces (see e.g.~\citet{zalinescu2002convex}). The space we consider is $\mathcal{M}(X)$, the space of signed, finite measures equipped with the topology of weak convergence, of which $\mathcal{P}(X)$ is a convex subset. $\mathcal{M}(X)$ is indeed locally convex and Hausdorff, and its dual space is $\mathcal{C}(X)$ (see e.g.~\citet{aliprantis1999infinite}, section 5.14).

    We now show that a maximizer $\varphi^*$ is an influence function. By the Fenchel--Moreau theorem, \[
        J(\mu) = J^{\star\star}(\mu) = \sup_{\varphi \in \mathcal{C}(X)} \Big[ \int_X \varphi \,d\mu - J^\star(\varphi) \Big],
    \] and \[
        J(\mu + \epsilon \chi) = \sup_{\varphi \in \mathcal{C}(X)} \Big[ \int_X \varphi \,d\mu + \epsilon\int_X \varphi \,d \chi - J^\star(\varphi) \Big].
    \]
    Because $J$ is differentiable, $\epsilon \mapsto J(\mu + \epsilon \chi)$ is differentiable, so by the envelope theorem \cite{milgrom2002envelope}, \[
        \frac{d}{d\epsilon}J(\mu + \epsilon \chi) \Big|_{\epsilon = 0} = \int_X \varphi^*\,d\chi,
    \]  so that $\varphi^*$ is an influence function by \autoref{thm:influence}.
    
    The abuse of notation stems from the fact that not all influence functions are maximizers. This is true, though, if $J(\mu) = \infty$ if $\mu \not\in \mathcal{P}(X)$: \begin{align*}
        &\int_X \Psi_\mu \,d\mu - J^\star(\Psi_\mu ) \\
            &\qquad= \int_X \Psi_\mu\,d\mu-\sup_{\nu \in \mathcal{P}(X)} \Big[ \int_X \Psi_\mu \,d\nu - J(\nu) \Big] \\
            &\qquad= \inf_{\nu \in \mathcal{P}(X)} \Big[ {-\int_X \Psi_\mu\,d(\nu-\mu)} + J(\nu) \Big] \\
            &\qquad= \inf_{\nu \in \mathcal{P}(X)} \Big[ {-\frac{d}{d\epsilon} J(\mu + \epsilon(\nu - \mu)) \Big|_{\epsilon=0}} + J(\nu) \Big] \\
            &\qquad \ge J(\mu),
    \end{align*} since the convex function $f(\epsilon) = J(\mu + \epsilon(\nu - \mu))$ lies above its tangent line: \[
         f(1) \ge f(0) + 1 \cdot f'(0).
    \] Since $J(\mu) = J^{\star\star}(\mu)$, we have that \[
        \int_X \Psi_\mu\,d\mu - J^\star(\Psi_\mu) \ge \sup_{\varphi \in \mathcal{C}(X)} \Big[\int_X \varphi\,d\mu - J^\star(\varphi) \Big].
    \]
\end{proof}

\autoref{thm:fenchel} motivates the following influence function approximation strategy: model $\varphi : X \to \R$ with a neural network and train it using stochastic gradient ascent on the objective $\phi \mapsto \E_{x\sim\mu}[ \varphi_\phi(x)] - J^\star(\varphi_\phi)$. The trained neural network is then an approximation to $\Psi_\mu$ suitable for use in the descent step of PFD. Under this approximation scheme, PFD can be concisely expressed as the saddle-point problem \[
    \inf_\mu \sup_\varphi \big[ \E_{x \sim \mu} [\varphi(x)] - J^\star(\varphi) \big],
\] where the inner supremum solves for the influence function (the differentiation step of PFD), and the outer infimum descends the linear approximation $\E_{x\sim \mu}[\varphi(x)]$ (the descent step of PFD), noting that $J^\star(\varphi)$ is a constant w.r.t.~$\mu$. This procedure is highly reminiscent of adversarial training \cite{goodfellow2014generative}; for this reason, we call PFD with this approximation scheme based on convex duality \textbf{adversarial PFD}. PFD therefore explains the prevalence of adversarial training as a deep learning technique and extends its applicability to any convex probability functional.

In the following sections, we demonstrate that PFD provides a broad conceptual framework for understanding a wide range of existing machine learning algorithms.

\begin{toappendix}
The following lemma will come in handy in our computations.
\begin{lemma} \label{thm:representation}
Suppose $J : \mathcal{M}(X) \to \overline{\R}$ has a representation \[
    J(\mu) = \sup_{\varphi \in \mathcal{C}(X)} \Big[\int_X \varphi\,d\mu - K(\varphi)\Big],
\] where $K : \mathcal{C}(X) \to \overline{\R}$ is proper, convex, and lower semicontinuous. Then $J^\star = K$.
\end{lemma}
\begin{proof}
By definition of the convex conjugate, $J = K^\star$. Then $J^\star = K^{\star \star} = K$, by the Fenchel--Moreau theorem.
\end{proof}
We note that when applying this lemma, we will often implicitly define the appropriate extension of $J$ to $\mathcal{M}(X)$ to be $J(\mu) = \sup_{\varphi \in \mathcal{C}(X)} [\int \varphi\,d\mu - K(\varphi)]$. The exact choice of extension can certainly affect the exact form of the convex conjugate; see \citet{ruderman2012tighter} for one example of this phenomenon.
\end{toappendix}

\section{Generative Adversarial Networks}
Generative adversarial networks (GANs) are a technique to train a parameterized probability measure $\mu$ to mimic a data distribution $\nu$. There are many variants of the GAN algorithm. They typically take the form of a saddle-point problem, and it is known that many of them correspond to the minimization of different divergences $D(\mu || \nu)$. We complete the picture by showing that many GAN variants could have been derived as instances of PFD applied to different divergences.

\subsection{Minimax GAN}
\citet{goodfellow2014generative} originally proposed the following saddle-point problem
\[
    \inf_\mu \sup_D \tfrac{1}{2} \E_{x\sim\nu} [\log D(x)] + \tfrac{1}{2}\E_{x\sim\mu}[\log(1 - D(x))] .
\]
The interpretation of this \emph{minimax GAN} problem is that the \emph{discriminator} $D$ learns to classify between fake samples from $\mu$ and real samples from $\nu$ via a binary classification loss, while the \emph{generator} $\mu$ is trained to produce counterfeit samples that fool the classifier. It was shown that the value of the inner optimization problem equals $D_{\mathrm{JS}}(\mu || \nu) - \log 2$, where \[
    D_{\mathrm{JS}}(\mu || \nu) = \tfrac{1}{2} D_{\mathrm{KL}}(\mu || \tfrac{1}{2}\mu + \tfrac{1}{2}\nu) + \tfrac{1}{2} D_{\mathrm{KL}}(\nu || \tfrac{1}{2}\mu + \tfrac{1}{2}\nu)
\] is the Jensen--Shannon divergence, and therefore the problem corresponds to training $\mu$ to minimize the divergence between $\mu$ and $\nu$. As a practical algorithm, simultaneous stochastic gradient descent steps are performed on the discriminator's parameters $\phi$ and the generator's parameters $\theta$ using the two loss functions \begin{equation}
    \begin{cases}
    \phi \mapsto -\tfrac{1}{2} \E_{x\sim\nu} [\log D_\phi(x)] - \tfrac{1}{2}\E_{x\sim\mu_\theta}[\log(1 - D_\phi(x))], \\
    \theta \mapsto  \tfrac{1}{2}\E_{x\sim\mu_\theta}[\log(1 - D_\phi(x))],
    \end{cases}    \label{eq:mmgan}
\end{equation} where $D_\phi$ and $\mu_\theta$ are parameterized with neural networks.

Our unifying result is the following:
\begin{proposition}
Adversarial PFD on the Jensen--Shannon divergence objective \[
    J_{\mathrm{JS}}(\mu) = D_{\mathrm{JS}}(\mu || \nu).
\] yields the minimax GAN algorithm \eqref{eq:mmgan}.
\end{proposition}
That is, the minimax GAN could have been derived mechanically and from first principles as an instance of adversarial PFD. To build intuition, we note that the discriminator plays the role of the approximate influence function:
\begin{propositionrep}
Suppose $\mu$ has density $p(x)$ and $\nu$ has density $q(x)$. Then the influence function for $J_{\mathrm{JS}}$ is \[
    \Psi_{\mathrm{JS}}(x) = \frac{1}{2} \log \frac{p(x)}{p(x)+q(x)}.
\]
\end{propositionrep}
\begin{proof}
    The result follows from \autoref{thm:influence}:
    \begin{align*}
        &\frac{d}{d\epsilon} J_{\mathrm{JS}}(\mu + \epsilon \chi) \Big|_{\epsilon = 0} \\
            &\qquad= \frac{1}{2}\int_X \frac{d}{d\epsilon}\Big[ (p+\epsilon\chi) \log \frac{p+\epsilon\chi}{\frac{1}{2}(p+\epsilon\chi)+ \frac{1}{2}q}  \\
            &\qquad\qquad\qquad\qquad  + q\log \frac{q}{\frac{1}{2}(p+\epsilon \chi) + \frac{1}{2}q}\Big]_{\epsilon=0} dx \\
            &\qquad= \frac{1}{2}\int_X\Big[ \log \frac{p}{\frac{1}{2}p+\frac{1}{2}q} + 1 - \frac{p}{p+q} - \frac{q}{p+q} \Big]\chi\, dx \\
            &\quad = \frac{1}{2}\int_X \Big[\log \frac{p}{p+q} + \log 2\Big] \chi\,dx.
    \end{align*}
\end{proof}
Recall that in the minimax GAN, the optimal discriminator $D^*$ satisfies $D^*(x) = \frac{q(x)}{p(x)+q(x)}$, so the influence function $\Psi_{\mathrm{JS}}(x) = \frac{1}{2}\log(1 - D^*(x))$ is approximated using the learned  discriminator.

Now, we rederive the minimax GAN problem \eqref{eq:mmgan} as a form of adversarial PFD. We compute:
\begin{propositionrep}
The convex conjugate of $J_{\mathrm{JS}}$ is \[
    J_{\mathrm{JS}}^\star(\varphi) = -\tfrac{1}{2}\E_{x\sim\nu}[ \log(1 - e^{2\varphi(x)+\log 2}) ] - \tfrac{1}{2}\log 2.
\]
\end{propositionrep}
\begin{proof}
\begin{align*}
    J_{\mathrm{JS}}^\star(\varphi)
        &= \sup_{\mu \in \mathcal{M}(X)} \Big[ \int_X \varphi \,d\mu - J_{\mathrm{JS}}(\mu) \Big] \\
        &= \sup_{p} \int_X \Big[ \varphi p - \frac{1}{2}p\log \frac{p}{\frac{1}{2}p+\frac{1}{2}q} - \frac{1}{2}q\log\frac{q}{\frac{1}{2}p+\frac{1}{2}q}\Big]\,dx.
\end{align*}
Setting the integrand's derivative w.r.t.~$p$ to $0$, we find that pointwise, the optimal $p$ satisfies \[
    \varphi = \frac{1}{2}\log \frac{p}{\frac{1}{2}p+\frac{1}{2}q}.
\] We eliminate $p$ in the integrand. Notice that the first two terms in the integrand cancel after plugging in $p$. Since \[
    \frac{q}{\frac{1}{2}p + \frac{1}{2}q} = 2\Big(1 - \frac{p}{p+q}\Big) = 2(1 - 2 e^{2\varphi}),
\] we obtain that \[
    J_{\mathrm{JS}}^\star(\varphi) = -\frac{1}{2}\int_X q \log(1 - 2e^{2\varphi}) \,dx - \frac{1}{2}\log 2.
\]
\end{proof}
\autoref{thm:fenchel} yields the representation \[
    \Psi_{\mathrm{JS}} = \argmax_{\varphi \in \mathcal{C}(X)} \Big[ \E_{x\sim\mu}[ \varphi(x)] +\tfrac{1}{2}\E_{x\sim\nu} [\log(1 - e^{2\varphi(x) + \log 2})] \Big],
\] an ascent step on which is the $\phi$-step in \eqref{eq:mmgan} with the substitution $\varphi = \frac{1}{2}\log (1-D) - \frac{1}{2}\log 2$. The descent step corresponds to updating $\mu$ to decrease the linear approximation $\E_{x \sim \mu} [\varphi(x)]$, which corresponds to the $\theta$-step in \eqref{eq:mmgan}. In fact, a similar argument can be applied to the $f$-GANs of \citet{nowozin2016f}, which generalize the minimax GAN. The observation that $f$-GANs (and hence the minimax GAN) can be derived through convex duality was also noted by \citet{farnia2018convex}.

\subsection{Non-saturating GAN}
\citet{goodfellow2014generative} also proposed an alternative to \eqref{eq:mmgan} called the \emph{non-saturating GAN}, which prescribes descent steps on \[  
    \begin{cases}
    \phi \mapsto -\frac{1}{2} \E_{x\sim\nu}[ \log D_\phi(x)] - \tfrac{1}{2}\E_{x\sim\mu_\theta}[\log(1 - D_\phi(x))], \\
    \theta \mapsto -\frac{1}{2}\E_{x\sim\mu_\theta}[\log D_\phi(x) ].
    \end{cases}
\] In the step on the generator's parameters $\theta$, the $\log(1-D_\phi)$ in the minimax GAN has been replaced with ${-\log D_\phi}$. This heuristic change prevents gradients to $\theta$ from converging to $0$ when the discriminator is too confident, and it is for this reason that the loss for $\theta$ is called the non-saturating loss. 

We consider a slightly modified problem, in which the original minimax loss and the non-saturating loss are summed (and scaled by a factor of $2$):
\begin{equation} \label{eq:nsgan}
     \begin{cases}
    \phi \mapsto -\frac{1}{2} \E_{x\sim\nu}[ \log D_\phi(x)] - \tfrac{1}{2}\E_{x\sim\mu_\theta}[\log(1 - D_\phi(x))],  \\
    \theta \mapsto - \E_{x\sim\mu_\theta}[ \log D_\phi(x)] + \E_{x\sim\mu_\theta} [\log(1 - D_\phi(x))] .
    \end{cases}
\end{equation} 

This also prevents gradients to $\theta$ from saturating, achieving the same goal as the non-saturating GAN. \citet{huszar2016alternative} and \citet{arjovsky2017towards} recognize that this process minimizes $D_{\text{KL}}(\mu||\nu)$.\footnote{The derivation of \citet{huszar2016alternative} omits showing that the dependence of $\frac{q(x)}{p_\theta(x)}$ on $\theta$ can be ignored, but the result is proved by Theorem 2.5 of \citet{arjovsky2017towards}. We remark that this result can be seen as a corollary of \autoref{thm:chain} and \autoref{thm:nsganderivative}.}

We claim the following:
\begin{proposition}
PFD on the reverse Kullback--Liebler divergence objective \[
    J_{\mathrm{NS}}(\mu) = D_{\mathrm{KL}}(\mu||\nu),
\] using the binary classification likelihood ratio estimator to approximate the influence function, yields the modified non-saturating GAN optimization problem \eqref{eq:nsgan}.
\end{proposition}

\begin{propositionrep} \label{thm:nsganderivative}
Suppose $\mu$ has density $p(x)$ and $\nu$ has density $q(x)$. The influence function for $J_{\mathrm{NS}}$ is \[
    \Psi_{\mathrm{NS}}(x) = \log \frac{p(x)}{q(x)}.
\]
\end{propositionrep}
\begin{proof}
The result follows from \autoref{thm:influence}:
\begin{align*}
    &\frac{d}{d\epsilon} J_{\mathrm{NS}}(\mu + \epsilon \chi) \Big|_{\epsilon=0}\\
        &\qquad= \frac{d}{d\epsilon} \int_X (p+\epsilon\chi) \log \frac{p+\epsilon\chi}{q}\,dx \Big|_{\epsilon=0} \\
        &\qquad= \int_X \Big[ \chi \log \frac{p}{q} + \chi\Big]\,dx \\
        &\qquad= \int_X \Big[ \log \frac{p}{q} + 1\Big]\,d\chi \\
        &\qquad= \int_X \Big[ \log \frac{p}{q} \Big]\,d\chi.
\end{align*}
\end{proof}

Now, because the binary classification loss \begin{equation}
    D \mapsto -\tfrac{1}{2} \E_{x\sim\nu}[ \log D(x)] - \tfrac{1}{2}\E_{x\sim\mu_\theta}[\log(1 - D(x))] \label{eq:logisticregression},
\end{equation} is minimized by $D(x) = \frac{q(x)}{p(x) + q(x)}$, one estimator for $\Psi_{\mathrm{NS}}$ is simply \[
     \Psi_{\mathrm{NS}}(x) \approx \log \frac{1 - D_\phi(x)}{D_\phi(x)},
\] where $\phi$ is updated as in the $\phi$-step in \eqref{eq:nsgan}. With this approximation scheme, the differentiation step and the descent step in PFD correspond exactly to the $\phi$-step and $\theta$-step respectively in \eqref{eq:nsgan}. Once again, the discriminator serves to approximate the influence function.

\subsection{Wasserstein GAN}
\citet{arjovsky2017wasserstein} propose solving the following saddle-point problem \[
    \inf_\mu \sup_{||D||_L \le 1} \big[ \E_{x\sim\mu}[D(x)] - \E_{x\sim\nu}[D(x)] \big],  
\] where $||D||_L$ denotes the Lipschitz constant of $D$. The corresponding practical algorithm amounts to simultaneous descent steps on
\begin{equation}
\begin{cases}
    \phi \mapsto \E_{x\sim\mu_\theta}[D_\phi(x)] - \E_{x\sim\nu}[D_\phi(x)],  \\
    \theta \mapsto - \E_{x\sim\mu_\theta}[D_\phi(x)], 
\end{cases}\label{eq:wgan}
\end{equation} where $D_\phi$ is reprojected back to the space of $1$-Lipschitz functions after each $\phi$-step. Here, $\mu_\theta$ is again the generator, and $D_\phi$ is the discriminator, sometimes called the critic. This algorithm is called the Wasserstein GAN algorithm, so named because this algorithm approximately minimizes the $1$-Wasserstein distance $W_1(\mu, \nu)$; the motivation for the $\phi$-step in \eqref{eq:wgan} is so that the discriminator learns the \emph{Kantorovich potential} that describes the optimal transport from $\mu$ to $\nu$. See e.g.~\citet{villani2008optimal} for the full optimal transport details.

We claim that the Wasserstein GAN too is an instance of PFD, and once again, the discriminator plays the role of approximate influence function:
\begin{proposition}
Adversarial PFD on the Wasserstein distance objective \[
    J_{\mathrm{W}}(\mu) = W_1(\mu, \nu)
\] yields the Wasserstein GAN algorithm \eqref{eq:wgan}.
\end{proposition}
\begin{propositionrep} \label{thm:wganderivative}
The influence function for $J_{\mathrm{W}}$ is the Kantorovich potential corresponding to the optimal transport from $\mu$ to $\nu$.
\end{propositionrep}
\begin{proof}
    See \citet{santambrogio2015functionals}, Proposition 7.17.
\end{proof}

We remark that the gradient computation in Theorem 3 of \citet{arjovsky2017wasserstein} is a corollary of \autoref{thm:chain} and \autoref{thm:wganderivative}. Now, we show that the Wasserstein GAN algorithm can be derived mechanically via convex duality. The connection between the Wasserstein GAN and convex duality was also observed by \citet{farnia2018convex}.
\begin{propositionrep}
The convex conjugate of $J_{\mathrm{W}}$ is \[
    J_{\mathrm{W}}^\star(\varphi)
        = \E_{x\sim\nu}[\varphi(x)] + \{ ||\varphi||_L \le 1 \}.
\]
\end{propositionrep}
We use the notation $\{ A \}$ to denote the convex indicator function, which is $0$ if $A$ is true and $\infty$ if $A$ is false.
\begin{proof}
Using Kantorovich--Rubinstein duality, we have that \begin{align*}
    J_{\mathrm{W}}(\mu) 
        &= \sup_{||\varphi||_L \le 1} \Big[ \int_X \varphi\,d\mu - \int_X \varphi\,d\nu \Big] \\
        &= \sup_{\varphi} \Big[ \int_X \varphi\,d\mu - \int_X \varphi\,d\nu - \{ ||\varphi||_L \le 1 \}\Big],
\end{align*} where we use the notation \[
    \{ A \} = \begin{cases}
        0 & \text{$A$ is true,}\\
        \infty & \text{$A$ is false}.
    \end{cases}
\]
By \autoref{thm:representation}, \[
    J_{\mathrm{W}}^\star(\varphi) = \int_X \varphi\,d\nu + \{ ||\varphi||_L \le 1 \}.
\] 
\end{proof}

\autoref{thm:fenchel} yields the representation \[
    \Psi_{\mathrm{W}}
    = \argmax_{\varphi \in \mathcal{C}(X)} \big[ \E_{x\sim\mu_\theta}[\varphi(x)] - \E_{x\sim\nu}[\varphi(x)] - \{ ||\varphi||_L \le 1 \} \big]. 
\] The adversarial PFD differentiation step therefore corresponds exactly to the $\phi$-step in \eqref{eq:wgan}, and the PFD descent step is exactly the $\theta$-step in \eqref{eq:wgan}.

\section{Variational Inference}
In Bayesian inference, the central object is the posterior distribution \[
    p(z|x) = \frac{p(x|z)p(z)}{p(x)} = \frac{p(x|z)p(z)}{\int p(x|z)p(z) \,dz},
\] where $x$ is an observed datapoint, $p(x|z)$ is the likelihood, $p(z)$ is the prior. Unfortunately, the posterior is difficult to compute due to the presence of the integral. Variational inference therefore reframes this computation as an optimization problem in which a \emph{variational posterior} $q(z)$ approximates the true posterior by solving \[
    \inf_q D_{\text{KL}}(q(z) || p(z|x)).
\]

\subsection{Black-box variational inference}
This objective is not directly optimizable, due to the presence of the intractable $p(z|x)$ term. The tool of choice for variational inference is the \emph{evidence lower bound} (ELBO), which rewrites \[
    D_{\text{KL}}(q(z) || p(z|x)) = \log p(x) - \underbrace{\E_{z \sim q(z)}\Big[\log \frac{p(x|z)p(z)}{q(z)}\Big]}_{\text{ELBO}}.  
\] Because $\log p(x)$ is fixed, we may maximize the ELBO to minimize the KL divergence. The advantage of doing so is that all the terms inside the expectation are now tractable to evaluate, and thus the expectation may be approximated through Monte Carlo sampling. This leads to the following practical algorithm, namely stochastic gradient descent on the objective \begin{equation}
    \theta \mapsto -\E_{z\sim q_\theta(z)}\Big[ \log \frac{p(x|z)p(z)}{q_\theta(z)} \Big].  \label{eq:bbvi}
\end{equation} This is called black-box variational inference \cite{ranganath2014black}. \citet{roeder2017sticking} later recognized that ignoring the $\theta$-dependence of the term in the expectation yields the same gradients in expectation; it is this variant that we consider. Our unification result is the following:
\begin{proposition}
PFD on the variational inference objective \[
    J_{\mathrm{VI}}(q) =  D_{\mathrm{KL}}(q(z)||p(z|x)),
\] using exact influence functions, yields the black-box variational inference algorithm \eqref{eq:bbvi}.
\end{proposition}

In fact, the influence function turns out to be precisely the inside of the negative ELBO bound:
\begin{propositionrep}
\label{thm:viderivative}
The influence function for $J_{\mathrm{VI}}$ is
    \[
        \Psi_{\mathrm{VI}}(z) = \log \frac{q(z)}{p(x|z)p(z)}.
    \]
\end{propositionrep}
\begin{proof}
The result follows from \autoref{thm:influence}:
\begin{align*}
    &\frac{d}{d\epsilon} J_{\mathrm{VI}}(q + \epsilon \chi) \Big|_{\epsilon=0}\\
        &\qquad= \frac{d}{d\epsilon} \int (q(z)+\epsilon\chi(z)) \log \frac{q(z)+\epsilon\chi(z)}{p(z|x)}\,dz \Big|_{\epsilon=0} \\
        &\qquad= \int \Big[ \chi(z) \log \frac{q(z)+\epsilon\chi(z)}{p(z|x)} + \chi(z)\Big]\,dz \Big|_{\epsilon=0} \\
        &\qquad= \int \Big[ \log \frac{q(z)}{p(z|x)} + 1\Big]\,\chi(z)\,dz  \\
        &\qquad= \int \Big[ \log \frac{q(z)}{p(x|z)p(z)} + \log p(x)+ 1\Big]\,\chi(z)\,dz  \\
        &\qquad= \int  \log \frac{q(z)}{p(x|z)p(z)}\,\chi(z)\,dz.
\end{align*}
\end{proof}

In this context, the influence function can be evaluated exactly, so the differentiation step of PFD may be performed without approximation. The descent step of PFD becomes exactly the descent step on $\theta$ of \eqref{eq:bbvi}, where the $\theta$-dependence of the term in the expectation is ignored. We remark that the argument of \citet{roeder2017sticking} that this $\theta$-dependence can be ignored can be seen as a corollary of \autoref{thm:chain} and \autoref{thm:viderivative}.

\subsection{Adversarial variational Bayes}
When the density function of the prior $p(z)$ or the variational posterior $q(z|x)$ is not available, adversarial variational Bayes \cite{mescheder2017adversarial} may be employed. Here, the quantity $\log \frac{q(z)}{p(z)}$ is approximated by a neural network $f_\phi(z)$ through a binary classification problem, much like \eqref{eq:logisticregression}. The resulting algorithm applies simultaneous descent steps on
\begin{equation}
    \begin{cases}
    \phi\mapsto -\E_{q_\theta(z)}[ \log \sigma(f_\phi(z))]-\E_{p(z)}[ \log (1 - \sigma(f_\phi(z)))] \\
    \theta\mapsto -\E_{q_\theta(z)} [-f_\phi(z) + \log p(x|z)].
    \end{cases} \label{eq:avb}
\end{equation}

This algorithm is another instance of PFD:
\begin{proposition} \label{thm:avb}
PFD on the variational inference objective $J_{\mathrm{VI}}$, using the binary classification likelihood ratio estimator to approximate the influence function, yields adversarial variational Bayes \eqref{eq:avb}.
\end{proposition}

It is easily seen that \[
    \Psi_{\mathrm{VI}}(z) = \log \frac{q(z)}{p(x|z)p(z)} \approx f_\phi(z) - \log p(x|z).
\] Therefore, the $\phi$-step of \eqref{eq:avb} is the differentiation step of PFD, and the $\theta$-step of \eqref{eq:avb} is the descent step. We remark that the gradient computation in Proposition 2 of \citet{mescheder2017adversarial} is a corollary of \autoref{thm:chain} and \autoref{thm:viderivative}.

% \subsection{Adversarial posterior distillation}
% Using convex duality to approximate the derivative presents some difficulty as sampling from the posterior $p(z|x)$ is required. \citet{wang2018adversarial} propose using MCMC to perform this sampling; the procedure then reduces to training a GAN to match $q(z)$ with $p(z|x)$, as described in the previous section.

\section{Reinforcement Learning}
In a Markov decision process, the distribution of states $s=(s_0, s_1, \ldots)$, actions $a=(a_1, a_2, \ldots)$, and rewards $r=(r_1, r_2, \ldots)$ is governed by the distribution \[
    \mathbb{P}(s, a, r) = p_0(s_0)\prod_{t=1}^\infty p(s_t, r_t|s_{t-1},a_t)\,\pi(a_t|s_{t-1}),
\] where $p_0(s)$ is an initial distribution over states, $p(s',r|s,a)$ gives the transition probability of arriving at state $s'$ with reward $r$ from a state $s$ taking an action $a$, and $\pi(a|s)$ is a policy that gives the distribution of actions taken when in state $s$. In reinforcement learning, we are interested in learning the policy $\pi(a|s)$ that maximizes the expected discounted reward $\E [\sum_{t=1}^\infty \gamma^{t-1} r_t]$, where $0 < \gamma < 1$ is a discount factor, while assuming we only have access to samples from $p_0$ and $p$.

\subsection{Policy iteration}
\label{sec:policyiteration}
Policy iteration \cite{howard1960dynamic,sutton1998introduction} is one scheme that solves the reinforcement learning problem. It initializes $\pi(s|a)$ arbitrarily and then cycles between two steps, policy evaluation and policy improvement. In the policy evaluation step, the state-action value function $Q^\pi(s,a)$ is computed. In the policy improvement step, the policy is updated to the greedy policy, the policy that at state $s$ takes the action $\argmax_a Q^\pi(s,a)$ with probability $1$.

Before we present our unification result, we introduce an arbitrary distribution over states $\pi(s)$ and consider the joint distribution $\pi(s,a)=\pi(s) \pi(a|s)$, so that $\pi$ is one probability distribution rather than one for every state $s$. Now:
\begin{proposition}
    PFD on the reinforcement learning objective \[
        J_{\mathrm{RL}}(\pi) = -\E \sum_{t=1}^\infty \gamma^{t-1} r_t,
    \] using exact influence functions and global minimization of the linear approximation, yields the policy iteration algorithm.
\end{proposition}

\begin{toappendix}
Proofs continue on the following page.
\onecolumn
\end{toappendix}

\begin{propositionrep} \label{thm:rlderivative}
The influence function for $J_{\mathrm{RL}}$ is
\[
    \Psi_{\mathrm{RL}}(s,a) = -\frac{\sum_{t=0}^\infty \gamma^t p^\pi_t(s)}{\pi(s)} (Q^\pi(s,a) - V^\pi(s)) ,
\] where $Q^\pi$ is the state-action value function, $V^\pi$ is the state value function, and $p^\pi_t$ is the marginal distribution of states after $t$ steps, all under the policy $\pi$.
\end{propositionrep}
\begin{proof}

First, we note that 
\begin{align*}
    &\frac{d}{d\epsilon}(\pi + \epsilon\chi)(a|s)\Big|_{\epsilon=0 } \\
        &\qquad= \frac{d}{d\epsilon} \frac{\pi(a,s) + \epsilon\chi(s,a)}{\pi(s) + \epsilon \chi(s)} \Big|_{\epsilon=0} \\
        &\qquad= \frac{\chi(s,a) - \chi(s)\pi(a|s)}{\pi(s)},
\end{align*} where we abuse notation to denote $\chi(s) = \int \chi(s,a')\,da'$.

We have \[
    -J_{\mathrm{RL}} = \E \Big[\sum_{t=1}^\infty \gamma^{t-1}r_t\Big],
\] or, plugging in the measure,
\[
-J_{\mathrm{RL}} = \int \sum_{t=1}^\infty \gamma^{t-1}r_t  \,p_0(s_0)  \prod_{j=1}^\infty p(s_j, r_j|s_{j-1},a_j)  \prod_{k=1}^\infty \pi(a_k|s_{k-1}) .
\] The integral is over all free variables; we omit them here and in the following derivation for conciseness.

In computing $\frac{d}{d\epsilon}J_{\mathrm{RL}}(\pi + \epsilon\chi)|_{\epsilon=0}$, the product rule dictates that a term appear for every $k$, in which $\pi(a_k|s_{k-1})$ is replaced with $\frac{d}{d\epsilon}(\pi+\epsilon\chi)(a_k|s_{k-1})|_{\epsilon=0}$. Hence:
\begin{align*}
    &-\frac{d}{d\epsilon}J_{\mathrm{RL}}(\pi + \epsilon\chi)\Big|_{\epsilon=0} \\
        &\qquad= \int \sum_{t=1}^\infty \gamma^{t-1}r_t  \,p_0(s_0) \prod_{j=1}^\infty p(s_j, r_j|s_{j-1},a_j)  \\
        &\qquad\qquad \times \sum_{k=1}^\infty \frac{\chi(s_{k-1},a_k) - \chi(s_{k-1}) \pi(a_k|s_{k-1})}{\pi(s_{k-1})} \prod_{\substack{\ell=1 \\ \ell \ne k}}^\infty \pi(a_\ell|s_{\ell-1}) \\
        &\qquad= \sum_{k=1}^\infty\int \sum_{t=1}^\infty \gamma^{t-1}r_t  \,p_0(s_0) \prod_{j=1}^\infty p(s_j, r_j|s_{j-1},a_j) \\
        &\qquad\qquad \times \frac{\chi(s_{k-1},a_k) - \chi(s_{k-1}) \pi(a_k|s_{k-1})}{\pi(s_{k-1})}  \prod_{\substack{\ell=1 \\ \ell \ne k}}^\infty \pi(a_\ell|s_{\ell-1}),
\end{align*} reordering the summations. Note that for $t < k$, the summand vanishes: 
\begin{align*}
    & \int  \prod_{j=k}^\infty p(s_j, r_j|s_{j-1},a_j)  \\
    &\qquad\qquad \times \big( \chi(s_{k-1},a_k) - \chi(s_{k-1}) \pi(a_k|s_{k-1}) \big)  \prod_{\ell=k+1}^\infty \pi(a_\ell|s_{\ell-1}) \\
    &\qquad= \int \big( \chi(s_{k-1},a_k) - \chi(s_{k-1}) \pi(a_k|s_{k-1}) \big) \\
    &\qquad= \int \big( \chi(s_{k-1}) - \chi(s_{k-1}) \big) \\
    &\qquad= 0,
\end{align*} since all the variables $a_k, r_k, s_{k}, a_{k+1}, r_{k+1}, s_{k+1}, \ldots$ integrate away to $1$. This yields:
\begin{align*}
    &-\frac{d}{d\epsilon}J_{\mathrm{RL}}(\pi + \epsilon\chi)\Big|_{\epsilon=0} \\
        &\qquad= \sum_{k=1}^\infty\int \sum_{t=k}^\infty \gamma^{t-1}r_t  \,p_0(s_0) \prod_{j=1}^\infty p(s_j, r_j|s_{j-1},a_j) \\
        &\qquad\qquad \times \frac{\chi(s_{k-1},a_k) - \chi(s_{k-1}) \pi(a_k|s_{k-1})}{\pi(s_{k-1})}  \prod_{\substack{\ell=1 \\ \ell \ne k}}^\infty \pi(a_\ell|s_{\ell-1}).
\end{align*}

Then, substituting the marginal distribution (note $s_{k-1}$ is not integrated) \[
    p_{k-1}^\pi(s_{k-1}) = \int \prod_{j=1}^{k-1} p(s_j, r_j|s_{j-1},a_j) \prod_{\ell=1}^{k-1} \pi(a_\ell|s_{\ell-1}),
\] we obtain
\begin{align*}
    &-\frac{d}{d\epsilon}J_{\mathrm{RL}}(\pi + \epsilon\chi)\Big|_{\epsilon=0} \\
        &\qquad= \sum_{k=1}^\infty\int \sum_{t=k}^\infty \gamma^{t-1}r_t  \,p^\pi_{k-1}(s_{k-1}) \prod_{j=k}^\infty p(s_j, r_j|s_{j-1},a_j) \\
        &\qquad\qquad \times \frac{\chi(s_{k-1},a_k) - \chi(s_{k-1}) \pi(a_k|s_{k-1})}{\pi(s_{k-1})}  \prod_{\ell =k+1}^\infty \pi(a_\ell|s_{\ell-1}).
\end{align*}

Let us rename the integration variables by decreasing their indices by $k-1$:
\begin{align*}
    &-\frac{d}{d\epsilon}J_{\mathrm{RL}}(\pi + \epsilon\chi)\Big|_{\epsilon=0} \\
        &\qquad= \sum_{k=1}^\infty\int \sum_{t=1}^\infty \gamma^{t+k-2}r_{t}  \,p^\pi_{k-1}(s_{0}) \prod_{j=1}^\infty p(s_j, r_j|s_{j-1},a_j) \\
        &\qquad\qquad \times \frac{\chi(s_{0},a_1) - \chi(s_{0}) \pi(a_1|s_{0})}{\pi(s_0)}  \prod_{\ell =2}^\infty \pi(a_\ell|s_{\ell-1}).
\end{align*}

Substituting in
\begin{align*}
    V^\pi(s_0) &= \int \sum_{t=1}^\infty \gamma^{t-1} r_t \prod_{j=1}^\infty p(s_j, r_j|s_{j-1},a_j)  \prod_{\ell =1}^\infty \pi(a_\ell|s_{\ell-1}), \\
    Q^\pi(s_0,a_1) &= \int \sum_{t=1}^\infty \gamma^{t-1} r_t \prod_{j=1}^\infty p(s_j, r_j|s_{j-1},a_j)  \prod_{\ell =2}^\infty \pi(a_\ell|s_{\ell-1}),
\end{align*} we obtain 
\begin{align*}
    &-\frac{d}{d\epsilon}J_{\mathrm{RL}}(\pi + \epsilon\chi)\Big|_{\epsilon=0} \\
        &\qquad= \sum_{k=1}^\infty\int \gamma^{k-1}  \,p^\pi_{k-1}(s_{0})  \frac{Q^\pi(s_0,a_1)\chi(s_{0},a_1) - V^\pi(s_0)\chi(s_{0})}{\pi(s_0)} .
\end{align*}

Finally, by \autoref{thm:influence}, we obtain that \[
    \Psi_{\mathrm{RL}}(s,a) = -\frac{\sum_{k=0}^\infty \gamma^k p^\pi_k(s)}{\pi(s)} (Q^\pi(s,a) - V^\pi(s)).
\]
\end{proof}
The descent step of PFD corresponds to taking a step on $\pi_\theta(s,a)=\pi(s)\pi_\theta(a|s)$ to decrease the linear approximation \[
    \theta \mapsto -\E_{\pi_\theta(s,a)}\Big[ \frac{\sum_{t=0}^\infty \gamma^t p^\pi_t(s)}{\pi(s)} (Q^\pi(s,a) - V^\pi(s)) \Big].
\] Setting $d^\pi(s) = (1-\gamma)\sum_{t=0}^\infty \gamma^t p^\pi_t(s)$, this simplifies to either \begin{align}
\label{eq:rllinapprox}
    \theta &\mapsto -\frac{1}{1-\gamma} \E_{d^\pi(s)}\E_{\pi_\theta(a|s)} [Q^\pi(s,a) - V^\pi(s)], \\
\label{eq:rllinapproxq}
    \theta &\mapsto -\frac{1}{1-\gamma} \E_{d^\pi(s)}\E_{\pi_\theta(a|s)} [Q^\pi(s,a)] + \text{constant.}
\end{align}
The most naive way to decrease \eqref{eq:rllinapproxq} is to globally minimize it. This corresponds to setting $\pi_\theta(a|s)$ to be the greedy policy. Hence, the evaluation of $Q^\pi(s,a)$ in policy iteration corresponds exactly to computing the influence function in the differentiation step of PFD, and the greedy policy update corresponds to applying the descent step.

\subsection{Policy gradient and actor-critic}
Policy iteration exactly computes the linear approximation and nonparametrically minimizes it. Now we consider algorithms in which the policy is parameterized and the descent step is taken using a gradient step on \eqref{eq:rllinapprox} or \eqref{eq:rllinapproxq}. If this approach is taken, there is a lot of flexibility in how the influence function can be approximated, but generally speaking, the result is an actor-critic method \cite{konda2000actor,sutton2000policy}, which describes a class of algorithms that approximates the value function of the current policy and then takes a gradient step on the parameters of the policy using the estimated value function. We claim:

\begin{proposition}
    Approximate PFD on the reinforcement learning objective $J_{\mathrm{RL}}$, where the influence function is estimated using, for example, Monte Carlo, least squares, or temporal differences, yields an actor-critic algorithm.
\end{proposition}

There is a huge number of possible approximations to the influence function; we list several and their corresponding algorithms. The simplest algorithm is the policy gradient algorithm, also known as REINFORCE \cite{williams1992simple}, which directly uses a Monte Carlo estimate of $Q^\pi(s,a)$ as the influence function estimator. Stochastic value gradients \cite{heess2015learning} and the closely related deterministic policy gradient \cite{silver2014deterministic} fit a neural network to $Q^\pi(s,a)$ using a temporal difference update and use that as the influence function approximation; their use of a neural network makes them compatible with the reparameterization trick. Advantage actor-critic \cite{mnih2016asynchronous} estimates $Q^\pi(s,a) - V^\pi(s)$ by estimating $Q^\pi(s,a)$ using Monte Carlo and fitting a neural network to $V^\pi(s)$ using least squares. All of these algorithms are traditionally justified by the celebrated policy gradient theorem \cite{sutton2000policy}; we remark that this theorem is a corollary of \autoref{thm:chain} and \autoref{thm:rlderivative}.

\subsection{Dual actor-critic}
Because $J_{\mathrm{RL}}$ is not convex, adversarial PFD does not directly apply. However, the form of \autoref{thm:rlderivative} strongly suggests fixing the arbitrary distribution $\pi(s)$ to be the discounted marginal distribution of states $d^\pi(s)$. Closely related to the linear programming formulation of reinforcement learning \cite{puterman1994markov}, this choice turns out to convexify $J_{\mathrm{RL}}$, thus enabling the use of convex duality to approximate its influence function. We expect to obtain an adversarial formulation of reinforcement learning; one such formulation is the dual actor-critic algorithm \cite{dai2017boosting,chen2016stochastic}:
\begin{equation}
    \sup_{\pi} \inf_V\ (1-\gamma) \E_{p_0(s)}[V(s)] \\
    + \E_{\pi(s,a)} [\mathcal{A}V(s,a)],  \label{eq:dualactorcritic}
\end{equation} where $\mathcal{A} V (s,a) = \E_{p(s',r|s,a)}[r + \gamma V(s')] - V(s)$. Indeed:

\begin{proposition}
Adversarial PFD on the reinforcement learning objective $J_{\mathrm{RL}}$ yields the dual actor-critic algorithm \eqref{eq:dualactorcritic}.
\end{proposition}

\begin{propositionrep}
The convex conjugate of $J_{\mathrm{RL}}$ is \[
    J_{\mathrm{RL}}^\star(\varphi) = (1-\gamma)\E_{p_0(s)}V_\varphi(s) + \{ V_\varphi \text{ exists} \},
\] where $V_\varphi$ is the unique solution to $\varphi = -\mathcal{A} V_\varphi$, if it exists.
\end{propositionrep}
\begin{proof}
As mentioned in the text, we set the arbitrary distribution $\pi(s) =  (1-\gamma)\sum_{t=0}^\infty \gamma^t p^\pi_t(s)$. In doing so, $\pi(s,a)$ becomes a state-action \emph{occupancy measure} that describes the frequency of encounters of the state-action pair $(s,a)$ over trajectories governed by the policy $\pi(a|s)$. It is known that there is a bijection between occupancy measures $\pi(s,a)$ and policies $\pi(a|s)$ \cite{syed2008apprenticeship,ho2016generative}.

We can enforce this setting by redefining \[
    J_{\mathrm{RL}}(\pi) = -\E \sum_{t=1}^\infty \gamma^{t-1} r_t + \Big\{ \forall s: \pi(s) = (1-\gamma)\sum_{t=0}^\infty \gamma^t p^\pi_t(s) \Big\},
\] where again $\{ \cdot \}$ is the convex indicator function. This equation can be rewritten as \[
    J_{\mathrm{RL}}(\pi) = -\E_{\pi(s,a)} R(s,a) + \Big\{ \forall s':\ \pi(s') = (1-\gamma)p_0(s') + \gamma \E_{\pi(s,a)} p(s'|s,a)\Big\},
\]
where $R(s,a) = \E_{p(s',r|s,a)}[r]$. The constraint is known as the \emph{Bellman flow equation}. This formulation is convex, as it is the sum of an affine function and an indicator of a convex set (indeed, an affine subspace). 

We recall $-\varphi = \mathcal{A}V_\varphi$, where $\mathcal{A} V (s,a) = \E_{p(s',r|s,a)}[r + \gamma V(s')] - V(s)$. Now, $V_\varphi$ is uniquely defined by $\varphi$ if a solution to the equation exists. To see this, note that $V_\varphi$ is the fixed point of the Bellman operator $\mathcal{T}^a$ defined by \[
    (\mathcal{T}^a V)(s) = (R + \varphi)(s,a) + \gamma \E_{p(s'|s,a)} V(s'),
\] which is contractive and therefore has a unique fixed point. A representation of $V_\varphi$ may be obtained via fixed point iteration using $\mathcal{T}^a$ for an arbitrary action $a$: \[
    V_\varphi(s) = \lim_{k\to\infty }(\mathcal{T}^a)^k 0 =  \E^a \sum_{t=1}^\infty \gamma^{t-1} (R + \varphi)(s_t,a),
\] where the expectation is taken under the deterministic policy $a$.

We rewrite $J_{\mathrm{RL}}$ using a Lagrange multiplier $V(s)$
\begin{align*}
    J_{\mathrm{RL}}(\pi) 
        &= -\E_{\pi(s,a)}R(s,a) + \sup_{V} \int V(s') \Big[ \pi(s') - (1-\gamma)p_0(s') - \gamma \E_{\pi(s,a)} p(s'|s,a)\Big]\,ds' \\
        &= \sup_{V}  -\E_{\pi(s,a)}R(s,a) + \E_{\pi(s)} V(s) - (1-\gamma)\E_{p_0(s)} V(s) - \gamma \E_{\pi(s,a)} \E_{p(s'|s,a)} V(s') \\
        &= \sup_{\varphi} \E_{\pi(s,a)} \varphi(s,a) - (1-\gamma) \E_{p_0(s)} V_\varphi(s) - \{ V_\varphi \text{ exists} \}.
\end{align*}

Note that $(1-\gamma) \E_{p_0(s)} V_\varphi(s) + \{ V_\varphi \text{ exists} \}$ is convex in $\varphi$; this stems from the fact that \[
      V_{\alpha \varphi + (1-\alpha)\varphi'} = \alpha V_\varphi + (1-\alpha) V_{\varphi'}.
\] The result follows from \autoref{thm:representation}.

\end{proof}

Using \autoref{thm:fenchel}, adversarial PFD therefore recovers \eqref{eq:dualactorcritic}: 
\begin{align*}
    &\inf_\pi \sup_\varphi \E_{\pi(s,a)}[\varphi(s,a)] - J^\star_{\mathrm{RL}}(\varphi) \\
        &\qquad= \inf_\pi \sup_\varphi \E_{\pi(s,a)}[-\mathcal{A}V_\varphi(s,a)] - (1-\gamma) \E_{p_0(s)} V_\varphi(s).
\end{align*}

\section{Conclusion}
This paper suggests several new research directions. First is the transfer of insight and specialized techniques from one domain to another. As just one example, in the context of GANs, \citet{arjovsky2017wasserstein} claim that constraining the discriminator to be $1$-Lipschitz improves the stability of the training algorithm -- could similarly constraining the analogous object in reinforcement learning, namely an approximation to the advantage function, lead to improved stability in deep reinforcement learning?

Moreover, the abstract viewpoint taken in this paper allows for the simultaneous development of new algorithms for GANs, variational inference, and reinforcement learning. General influence function approximation techniques in the spirit of convex duality could improve all three fields at once. More sophisticated descent techniques beyond gradient descent on parameterized probability distributions, such as Frank-Wolfe or trust-region methods, could improve learning or yield valuable convergence guarantees.

Finally, this paper unlocks the possibility of applying probability functional descent to new problems. In principle, the algorithm can be applied mechanically to any situation where one wants to optimize over probability distributions, possibly leading to new, straightforward ways to solve problems in, for example, mathematical finance, mean field games, or POMDPs. One could argue that the current excitement over deep learning began once researchers realized that to solve a problem, they could simply write a loss function and then rely on automatic differentiation and gradient descent to minimize it. We hope that probability functional descent provides a similarly turnkey solution for optimizing loss functions defined on probability distributions and leads to a similar burst of research activity. 

\section*{Acknowledgements}
We thank Rui Shu, Yang Song, Shengjia Zhao, Abubakar Abid, Andrea Zanette, Jordi Feliu Fab\`a, Jing An, Abeynaya Gnanasekaran, and Kailai Xu for helpful discussions. Support from NSF grants DMS-1720451 and DMS-1820942 is gratefully acknowledged by J.~Blanchet.


\bibliography{example_paper}
\bibliographystyle{icml2019}

\clearpage
\appendix
\twocolumn[
\thispagestyle{empty}{\center\baselineskip 18pt
                       \toptitlebar{\Large\bf Supplementary Material for Probability Functional Descent}\bottomtitlebar}
]

\section{Proofs and Computations}

\end{document}